\documentclass[letterpaper, conference]{IEEEtran}
\IEEEoverridecommandlockouts

\usepackage{array}
\usepackage{textcomp}
\usepackage{stfloats}
\usepackage{url}
\usepackage{verbatim}
\usepackage{graphicx}
\usepackage{times}

\usepackage{bm}
\usepackage[numbers]{natbib}
\usepackage{xcolor}
\usepackage{paralist}
\usepackage{comment}
\usepackage{footnote}
\usepackage{graphicx}
\usepackage{amsmath, amsfonts, amssymb, amsthm}
\usepackage[pagebackref=true,breaklinks=true,colorlinks=true, citecolor=magenta,linkcolor=orange,urlcolor=blue]{hyperref}
\usepackage[left=.78in,top=1in,right=.78in,bottom=.9in,headsep=0.25in,headheight=20pt,heightrounded]{geometry}

\newtheorem{property}{Property}

\newtheorem{remark}{Remark}

\newtheorem{theorem}{Theorem}
\newtheorem{corollary}{Corollary}

\DeclareGraphicsExtensions{.pdf,.jpeg,.png,.eps}

\def\BibTeX{{\rm B\kern-.05em{\sc i\kern-.025em b}\kern-.08em
		T\kern-.1667em\lower.7ex\hbox{E}\kern-.125emX}}
\usepackage[compact]{titlesec}

\pdfoutput=1    

\begin{document}
	\title{
	Lagrangian Properties and Control of Soft Robots Modeled with Discrete Cosserat Rods.
	} 
	
	\author{L{e}k{a}n M{o}l{u}, Shaoru Chen,  and Audrey Sedal
		\thanks{
			Lekan Molu and Shaoru Chen are with Microsoft Research NYC. Audrey Sedal is with McGill University's Department of Mechanical Engineering.
			Emails: \{lekanmolu, shaoruchen\}@microsoft.com, audrey.sedal@mcgill.ca.
			}
		}
	
	
	\maketitle
	\thispagestyle{empty}
	\pagestyle{empty}


\newcommand{\lb}[1]{\textcolor{light-blue}{#1}}
\newcommand{\bl}[1]{\textcolor{blue}{#1}}

\newcommand{\maybe}[1]{\textcolor{gray}{\textbf{MAYBE: }{#1}}}
\newcommand{\inspect}[1]{\textcolor{blue}{\textbf{CHECK THIS: }{#1}}}
\newcommand{\more}[1]{\textcolor{red}{\textbf{MORE: }{#1}}}
\renewcommand{\figureautorefname}{Fig.}
\renewcommand{\sectionautorefname}{$\S$}
\renewcommand{\equationautorefname}{equation}
\renewcommand{\subsectionautorefname}{$\S$}
\renewcommand{\chapterautorefname}{Chapter}

\newcommand{\cmt}[1]{{\footnotesize\textcolor{red}{#1}}}
\newcommand{\lekan}[2]{{\footnotesize\textcolor{purple}{LM: #1}}{#2}}
\newcommand{\gilbert}[2]{{\footnotesize\textcolor{blue}{GB: #1}}{#2}}
\newcommand{\audrey}[2]{{\footnotesize\textcolor{magenta}{AS: #1}}{#2}}
\newcommand{\ames}[2]{{\footnotesize\textcolor{pink}{AA: #1}}{#2}}
\newcommand{\todo}[1]{\textcolor{cyan}{TO-DO: #1}}
\newcommand{\review}[1]{\noindent\textcolor{red}{$\rightarrow$ #1}}
\newcommand{\response}[1]{\noindent{#1}}
\newcommand{\stopped}[1]{\color{red}STOPPED HERE #1\hrulefill}

\newcommand{\linkToPdf}[1]{\href{#1}{\blue{(pdf)}}}
\newcommand{\linkToPpt}[1]{\href{#1}{\blue{(ppt)}}}
\newcommand{\linkToCode}[1]{\href{#1}{\blue{(code)}}}
\newcommand{\linkToWeb}[1]{\href{#1}{\blue{(web)}}}
\newcommand{\linkToVideo}[1]{\href{#1}{\blue{(video)}}}
\newcommand{\linkToMedia}[1]{\href{#1}{\blue{(media)}}}
\newcommand{\award}[1]{\xspace} %

\newcounter{mnote}
\newcommand{\marginote}[1]{\addtocounter{mnote}{1}\marginpar{\themnote. \scriptsize #1}}
\setcounter{mnote}{0}
\newcommand{\ie}{i.e.\ }
\newcommand{\eg}{e.g.\ }
\newcommand{\cf}{cf.\ }
\newcommand{\yes}{\checkmark}
\newcommand{\no}{\ding{55}}

\newcommand{\flabel}[1]{\label{fig:#1}}
\newcommand{\seclabel}[1]{\label{sec:#1}}
\newcommand{\tlabel}[1]{\label{tab:#1}}
\newcommand{\elabel}[1]{\label{eq:#1}}
\newcommand{\alabel}[1]{\label{alg:#1}}
\newcommand{\fref}[1]{\cref{fig:#1}}
\newcommand{\sref}[1]{\cref{sec:#1}}
\newcommand{\tref}[1]{\cref{tab:#1}}
\newcommand{\eref}[1]{\cref{eq:#1}}
\newcommand{\aref}[1]{\cref{alg:#1}}

\newcommand{\bull}[1]{$\bullet$ #1}
\newcommand{\argmax}{\text{argmax}}
\newcommand{\argmin}{\text{argmin}}
\newcommand{\mc}[1]{\mathcal{#1}}
\newcommand{\bb}[1]{\mathbb{#1}}

\def\tidx{t}
\def\reline{\mathbb{R}}
\def\ren{\mathbb{R}^n}
\newcommand{\Note}[1]{}
\renewcommand{\Note}[1]{\hl{[#1]}}  


\def\kau{\mc{K}}
\def\particle{\bm{x}}
\def\materialresponse{\bm{G}}
\def\orthoggroup{{\textit{SO}}(3)}
\def\liegroup{{\mathbb{SE}}(3)}
\def\liealgebra{\mathfrak{se}(3)}
\def\identity{\bm{I}}
\def\adj{\text{ad}}
\def\Adj{\text{Ad}}
\def\pos{p}
\def\rot{R}
\def\vinput{u_s}
\def\ufast{u_f}
\def\abscissa{X}
\def\conf{\bm{g}}
\def\ones{\bm{1}}
\def\strain{\xi}
\def\twist{\eta}
\def\gencoord{q}
\def\hilbert{{H}}
\def\banach{B}
\def\pde{p.d.e.}
\def\cref{cf.\ }
\def\perturb{\epsilon}
\def\spacec{\textbf{C}}
\def\spacex{\textbf{X}}
\def\masscom{\mc{M}}
\def\masscore{\masscom^{\text{core}}}
\def\masspert{\masscom^{\text{pert}}}
\def\lyapweight{\phi}

	\begin{abstract}
	The characteristic ``in-plane" bending associated with soft robots' deformation make them preferred over rigid robots in sophisticated manipulation and movement tasks. Executing such motion strategies to precision in soft deformable robots and structures is however fraught with modeling and control challenges given their infinite degrees-of-freedom. Imposing \textit{piecewise constant strains} (PCS) across (discretized) Cosserat microsolids on the continuum material however, their dynamics become amenable to tractable mathematical analysis. While this PCS model handles the characteristic difficult-to-model ``in-plane" bending well, its Lagrangian properties are not exploited for control in literature neither is there a rigorous study on the  dynamic performance of multisection deformable materials  for ``in-plane" bending that guarantees steady-state convergence. In this sentiment, we first establish the PCS model's structural Lagrangian properties. Second, we exploit these for control on various strain goal states. Third, we benchmark our hypotheses against an Octopus-inspired robot arm under different constant tip loads. These induce non-constant ``in-plane" deformation and we regulate strain states throughout the continuum in these configurations. Our numerical results establish  convergence to  desired equilibrium throughout the continuum in all of our tests. Within the bounds here set, we conjecture that our methods can find wide adoption in the control of cable- and  fluid-driven  multisection soft robotic arms; and may be extensible to the (learning-based) control of deformable agents employed in simulated, mixed, or augmented reality.
	\newline 
	\newline 
	\textbf{Supplementary material} --- The codes for reproducing the experiments reported in this paper are available online: \href{https://github.com/robotsorcerer/dcm}{https://github.com/robotsorcerer/dcm}.
\end{abstract}

	\section{Introduction}
Soft robots are attracting wide adoption in the automation community owing to their improved bending, torsion, configurability, and compliance properties. These properties enable customizable solutions for assistive wearable devices~\cite{wearables1, wearables2}, robot grippers~\cite{grippers}, and mobile robots~\cite{Marchese}. 
In-plane bending, as a motion deformation strategy, is a particular advantage that serial soft arms execute better than rigid arms.  With this property, soft robot grippers can pick up delicate objects, adopt adept nonlinear motion strategies in complex workspaces \eg for reaching tasks. In embodied applications such as soft exogloves~\cite{ConorSoroGlove} and exosuits~\cite{ConorLamb}, they provide conformal deformation that aids improved walking gaits, and muscle rehabilitation and they conserve energy economy. 

The piecewise constant strain (PCS) Lagrangian dynamics derived by~\cite{RendaTRO18} is a reduced special Euclidean group-3 ($\liegroup$)  model for  compliant, serial  manipulators. Derived from Cosserat~\cite{CosseratBrothers} rod theory, it addresses torsion, in-plane, and out-of-plane (multi-) bending motions. 
The PCS model outperforms the common piecewise constant curvature (PCC)~\cite{RSSPCC} and the constant curvature variants~\cite{IanWalkerDiffKine} used outside of finite element modeling methods (FEM)~\cite{FEM}. Controllers deployed on the PCS model however ignore the geometric properties of the Lagrangian dynamics~\cite{RendaStatics, ThuruthelSoRo}. This makes the control equations complicated and destroys several properties of the Lagrangian dynamics  that are useful for control tasks.

\begin{figure}[tb!]
	\centering
	\begin{minipage}[b]{.5\textwidth}
		\includegraphics[angle=0,width=\textwidth]{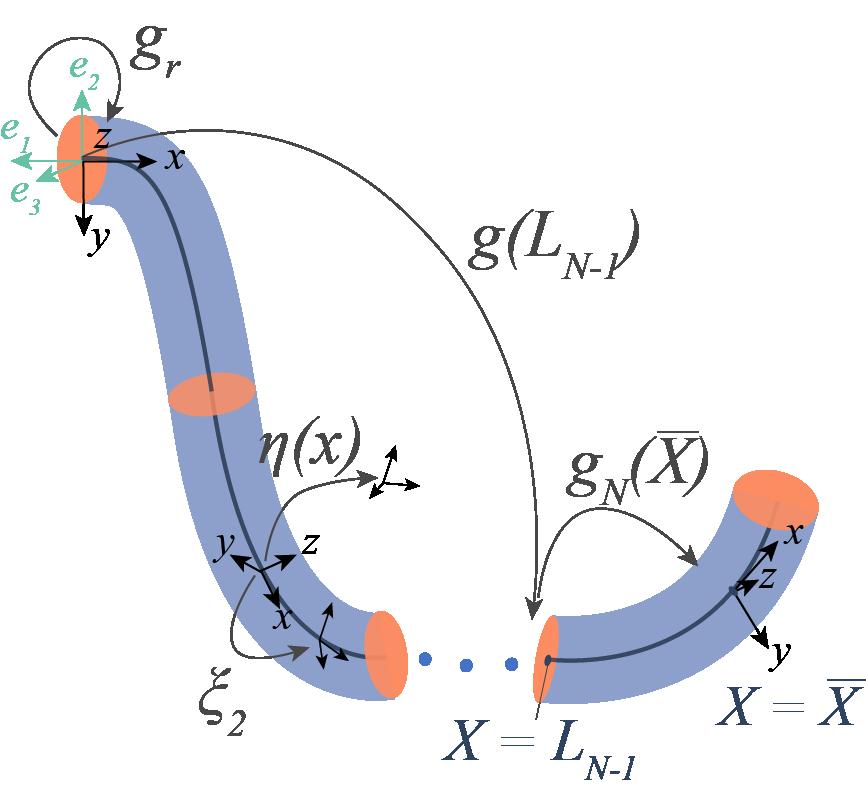} 
	\end{minipage} 
	\caption{Schematic of the configuration of an Octopus robot arm.}
\label{fig:pcs_kine}
\end{figure}
%

In this sentiment, considering the PCS model for soft multisection arms, we first establish the  structural properties of their Lagrangian dynamics. We then exploit these for the steady-state convergence analyses in various control proposals (under different operating model parameters)  and in different surrounding mediums using the well-established method of Lyapunov analysis. The Octopus robot~\cite{CeciliaOctopus}, a muscular hydrostat~\cite{e-book_on_morphological_computation_2014} with distributed deformation throughout its articulated arms, blends the interplay between continuum mechanics and sensorimotor control  well. Thus, we consider a single arm of the CyberOctopus~\cite{HierarchicalGazzolla} (configuration shown in \autoref{fig:pcs_kine}) to benchmark our controllers \textit{on the characteristic difficult-to-model-and-control in-plane bending deformation }. 

The rest of this paper is structured as follows: in \S \ref{sec:prelims}, we briefly introduce preliminaries relevant to the contributions machinery we present in  \S \ref{sec:structure}. Multivariable stabilizing feedback controllers for regulating the tip point and strain states are established in \S \ref{sec:pd_pid_controls}. We present numerical results in \S \ref{sec:results} and conclude the paper in \S \ref{sec:conclude}. 
	\section{Notations and Preliminaries}
\label{sec:prelims}
%
%
The strain field and strain twist vectors are respectively $\strain \in \reline^6$ and $\twist \in \bb{R}^6$.  The Lie algebra of  $\liegroup$ is  $\liealgebra$. 
All rotation are measured counterclockwise in corkscrew sense.  The continuum is reduced to a space curve defined by a material abscissa $\abscissa: [0, L]$, where $L$ is the robot's length in the reference configuration. 
For a configuration $\conf(\abscissa) \in \liegroup$, its adjoint and coadjoint representations are respectively $\Adj_{\conf}$, $\Adj_{\conf}^\star$ --- parameterized by a curve, the material abscissa $\abscissa$. The corresponding adjoint and coadjoint representation of the strain twist vector are $\adj_{\strain, \twist}$ and $\adj_{\strain, \twist}^\star$, respectively. In generalized coordinate, the joint vector of a soft robot is denoted $\gencoord(\strain) = [\strain_1^\top, \ldots, \strain_{n_\strain}^\top]^\top \in \reline^{6{n_\strain}}$. 

%
\subsection{SoRo Configuration}
%
The inertial frame is signified by the basis triad $(\bm{e}_1, \bm{e}_2, \bm{e}_3)$ (see \autoref{fig:pcs_kine}) and $\conf_r(X)$ is the transformation from the inertial to the manipulator's base frame.  For cable-driven arms, the point at which actuation occurs is labeled $\bar{\abscissa}$. The configuration matrix that parameterizes curve $X$ of length $L_n$ is denoted $\conf_{L_n}$. The cable runs through the $z$-axis in the (micro) body frame (x-axis in the spatial frame). As shown in \autoref{fig:pcs_kine}, 
the transformation from the base to the inertial frame is
\begin{align}
	\conf_r = \left(\begin{array}{cccc}
		0 & -1 & 0 & 0 \\
		1 & 0 & 0 & 0 \\
		0 & 0 & 1 & 0 \\
		0 & 0 & 0 & 1
	\end{array}\right).
	\label{eq:conf_ref}
\end{align} 
We adopt a linear viscoelastic constitutive model based on Kelvin-Voight internal forces \ie 
\begin{align}
	\mc{F}_i(\abscissa) = \Sigma(\xi - \xi_0) + \Upsilon \dot{\strain},
\end{align} 
with $\Sigma$ being the constant screw stiffness matrix,  $\Upsilon$ the viscosity matrix, and  the reference strain $\xi_0 = [0,0,0, 1,0,0]^\top$ in the upright configuration. The jacobian associated with $\abscissa$ is $J(\abscissa)$. Owing to paper length constraints we refer readers to~\citet{RendaTRO18} for further  definitions.
\subsection{Continuous Strain Vector and Twist Velocity Fields}
Let $\pos (X)$ describe the position vector of a microsolid on a soft body and let $\rot(X)$ denote the corresponding orientation matrix. Furthermore, let the pose $[\pos(X), \rot(X)]$ be parameterized by a material abscissa $\abscissa \in [0, L]$ for end nodes $0$ and $L$. Then, the robot's C-space, parameterized by a curve $g(\abscissa): \abscissa \rightarrow \liegroup$, is $
g(\abscissa) = \left(\begin{array}{cc}
\rot(\abscissa) & \pos(\abscissa) \\
0^\top & 1
\end{array}\right)$.
%
%
Suppose that $\varepsilon(\abscissa) \in \reline^3$ and $\gamma(\abscissa) \in \reline^3$ respectively denote the linear and angular strain components of the soft arm. Then, the arm's strain field is a state vector along the curve $g(\abscissa)$ defined as $\breve{\xi}(\abscissa) = g^{-1} \partial g/\partial \abscissa\triangleq g^{-1} \partial_x g$, and the velocity of $g(\abscissa)$ is the twist vector field $\breve{\eta}(\abscissa)$ defined as $\breve{\eta}(\abscissa) = g^{-1} \partial g / \partial t \triangleq g^{-1} \partial_t{g}$.
%
%
%
Read $\hat{\gamma}$: the anti-symmetric matrix representation of $\gamma$. Read $\breve{\xi}$: the isomorphism from $\xi \in \reline^6$, to its matrix  representation in  $\liealgebra$. 

\subsection{Discrete Cosserat System: Cosserat-Constitutive PDEs}
The PCS model assumes that the respective strain, $\strain_i$ for $n$ microsolid sections, $\{\mc{M}_i\}_{i=1}^n$, are constant for each discretized section of the arm. 
Using d'Alembert's principle, the generalized dynamics for the PCS model (\cf \autoref{fig:pcs_kine}) under external and actuation loads admits the weak form~\cite{RendaTRO18}
\begin{align}
&\underbrace{\left[\int_0^{L_N} J^T \mc{M}_a
	J d\abscissa\right]}_{M(q)}\ddot{q} + \underbrace{\left[\int_0^{L_N} J^T \adj_{J\dot{q}}^\star\mc{M}_a
	J d\abscissa\right]}_{C_1(q, \dot{q})}\dot{q} + \nonumber \\
&\underbrace{\left[\int_0^{L_N} J^T \mc{M}_a
	\dot{J} d\abscissa\right]}_{C_2(q, \dot{q})}\dot{q}
+ \underbrace{\left[\int_0^{L_N} J^T \mc{D}
	{J} \|J \dot{q}\|_pd\abscissa\right]}_{D(q, \dot{q})}\dot{q} \nonumber\\
& -\underbrace{(1-\rho_f/\rho) \left[\int_0^{L_N} J^T \mc{M} \Adj_{\conf}^{-1} d\abscissa\right]}_{N(q)}\Adj_{\conf_r}^{-1}\mc{G} - \underbrace{J(\bar{\abscissa})^T \mc{F}_p}_{F(q)} \nonumber \\
& - \underbrace{\int_0^{L_N} J^T \left[\nabla_x\mc{F}_i - \nabla_x\mc{F}_a + \adj_{\twist_n}^\star \left(\mc{F}_i - \mc{F}_a\right)\right]}_{\tau(q)}d\abscissa = 0,
\label{eq:loads_full_form}
\end{align}
where $\mc{F}_i(\abscissa)\triangleq \left(M_i^T, F_i^T \right)^T\equiv \partial_\abscissa \mathfrak{U}$ is the wrench of internal forces, $\bar{\mc{F}}_a(\abscissa)$ 
is the \textit{distributed} wrench of  actuation loads, and $\bar{\mc{F}}_e(\abscissa)$ 
is the external \textit{distributed} wrench of the applied forces. 
The screw mass inertia matrix  $\mc{M}(X) = \text{diag}\left(I_x, I_y, I_z, A, A, A\right)\rho$ for a body density $\rho$, sectional area $A$, bending, torsion, and second inertia operator $I_x, I_y, I_z$ respectively. 
In \eqref{eq:loads_full_form}, $\mc{M}_a = \mc{M} + \mc{M}_f$ is a lumped sum of the microsolid mass inertia operator, $\mc{M}$, and that of the added mass fluid, $\mc{M}_f$; $dX$ is the length of each section of the multi-robot arm; $\mc{D}(\abscissa)$ is the drag  matrix; $J(\abscissa)$ is the Jacobian operator; $\|\cdot\|_p$ is the translation norm of the expression contained therein; $\rho_f$ is the density of the fluid in which the material moves; $\rho$ is the body density; $\mc{G} = \left[0, 0, 0, -9.81, 0, 0\right]^T$ is the gravitational vector; and $\mc{F}_p$ is the applied wrench at  $\bar{\abscissa}$. In standard Newton-Euler form, equation \eqref{eq:loads_full_form}, is
\begin{align}
\begin{split}
M(q) \ddot{q} &+ \left[C_1 (q, \dot{q}) + C_2(q, \dot{q})\right]\dot{q} = \tau(q) + F(q) \\ 
& \qquad \qquad + N(q)\Adj_{\conf_r}^{-1} \mc{G} - D(q, \dot{q}) \dot{q}.
\end{split}
\label{eq:newton-euler}
\end{align}
%

	\section{Structural Properties of the PCS Model}
\label{sec:structure}
We now establish the Lagrangian properties. 

\begin{theorem}[Structural properties of the kinetic equation]
	Equation \eqref{eq:newton-euler} satisfies the following properties:
	\begin{property}[Positive definiteness of the Inertia Operator]
		The inertia tensor $\mc{M}_a(\gencoord)$ is symmetric and positive definite. As a result $M(\gencoord)$ is symmetric and positive definite. 
		\label{ppty:psd}
	\end{property}
	\begin{proof}[Proof of Property \ref{ppty:psd}]		
		The jacobian, $J$, is injective by ~\cite[equation 20]{RendaTRO18}. Thus, property \ref{ppty:psd} follows from its definition. 
	\end{proof}
	\begin{property}[Boundedness of the Mass Matrix]
		The mass inertial matrix $M(\gencoord)$ is uniformly bounded from below by $m\identity_n$ where $m$ is a positive constant. 
		\label{ppty:bdedness}
	\end{property}
	\begin{proof}[Proof of Property \ref{ppty:bdedness}]
		This is a restatement of the lower boundedness of $M(\gencoord)$ for fully actuated n-degrees of freedom manipulators~\cite{Ortega2014}.
	\end{proof}
	\begin{property}[Skew symmetric property]
		The matrix 
		$	\dot{M}(\gencoord) - 2 \left[C_1(\gencoord, \dot{\gencoord}) + C_2(\gencoord, \dot{\gencoord})\right]$
		is skew-symmetric.
		\label{ppty:skewsem}
	\end{property}
	\begin{proof}[Proof of Property \ref{ppty:skewsem}]
		We have by Leibniz's rule that 
		\begin{align}
			\dot{M}(\gencoord) &= \dfrac{d}{dt}\left(\int_{0}^{L_N} J^T \mc{M}_a J d\abscissa\right) = \int_{0}^{L_N} \dfrac{\partial}{\partial t} \left(J^T \mc{M}_a J \right)d\abscissa \nonumber \\
			&\triangleq \int_{0}^{L_N} \left(\dot{J}^T \mc{M}_a J  + J^T\dot{\mc{M}}_a J  +  {J}^T \mc{M}_a \dot{J}\right) d\abscissa.
			\label{eq:mass_skewsem}
		\end{align}
		Therefore, $\dot{M}(\gencoord) - 2 \left[C_1(\gencoord, \dot{\gencoord}) + C_2(\gencoord, \dot{\gencoord})\right]$ becomes
		\begin{align}
			&\int_{0}^{L_N} \left(\dot{J}^\top \mc{M}_a J  + J^\top\dot{\mc{M}}_a J  +  {J}^\top \mc{M}_a \dot{J}\right) d\abscissa \nonumber \\
			&- 2\int_{0}^{L_N}  \left(J^\top \adj^\star_{J\dot{\gencoord}} \mc{M}_a J + J^\top \mc{M}_a \dot{J}\right) d\abscissa \\
			&\triangleq\int_{0}^{L_N} \left(\dot{J}^\top \mc{M}_a J  + J^\top\dot{\mc{M}}_a J  -  {J}^\top \mc{M}_a \dot{J}\right) d\abscissa \nonumber \\
			&\qquad - 2\int_{0}^{L_N} J^\top \adj^\star_{J\dot{\gencoord}} \mc{M}_a J d\abscissa.
			\label{eq:skewsem_lhs}
		\end{align}
		Similarly, $-\left[\dot{M}(\gencoord) - 2 \left[C_1(\gencoord, \dot{\gencoord}) + C_2(\gencoord, \dot{\gencoord})\right]\right]^\top$ expands as 
		\begin{align}
			&-\dot{M}^\top(\gencoord) + 2 \left[C_1^\top(\gencoord, \dot{\gencoord}) + C_2^\top(\gencoord, \dot{\gencoord})\right] = \nonumber\\
			& \int_{0}^{L_N} d\abscissa^\top \left(-J^\top\mc{M}_a \dot{J}   - J^\top\dot{\mc{M}}_a J  -  \dot{J}^\top \mc{M}_a {J}\right)  \nonumber \\
			&\qquad + 2\int_{0}^{L_N}  d\abscissa^\top\left(J^\top \mc{M}_a \adj_{J\dot{\gencoord}} {J} + \dot{J}^\top \mc{M}_a J\right)   \nonumber\\ 
			%
			&\triangleq \int_{0}^{L_N} \left({J}^\top \mc{M}_a \dot{J}-\dot{J}^\top \mc{M}_a J  - J^\top\dot{\mc{M}}_a J  \right) d\abscissa \nonumber \\
			&\qquad - 2\int_{0}^{L_N} J^\top \adj^\star_{J\dot{\gencoord}} \mc{M}_a J d\abscissa
			\label{eq:skewsem_rhs}
		\end{align}
		where the terms in equation \eqref{eq:skewsem_rhs} follow from the symmetry of the matrices that constitute the integrands. Inspecting \eqref{eq:skewsem_lhs} and  \eqref{eq:skewsem_rhs}, it is easy to see that their right hand sides verify the identity
		\begin{align}
			&\int_{0}^{L_N} \left(\dot{J}^\top \mc{M}_a J  + J^\top\dot{\mc{M}}_a J  -  {J}^\top \mc{M}_a \dot{J}\right) d\abscissa \nonumber \\
			&\quad - 2\int_{0}^{L_N} J^\top \adj^\star_{J\dot{\gencoord}} \mc{M}_a J d\abscissa = 2\int_{0}^{L_N} J^\top \adj^\star_{J\dot{\gencoord}} \mc{M}_a J d\abscissa - \nonumber \\
			&\qquad \int_{0}^{L_N} \left({J}^\top \mc{M}_a \dot{J}-\dot{J}^\top \mc{M}_a J  - J^\top\dot{\mc{M}}_a J  \right) d\abscissa  
		\end{align}
		or 
		\begin{align}
			\dot{M}(\gencoord) &- 2 \left[C_1(\gencoord, \dot{\gencoord}) + C_2(\gencoord, \dot{\gencoord})\right] = \nonumber \\
			&-\left[\dot{M}(\gencoord) - 2 \left[C_1(\gencoord, \dot{\gencoord}) + C_2(\gencoord, \dot{\gencoord})\right]\right]^\top.
		\end{align}
		
		\textit{A fortiori}, the skew symmetric property follows.
	\end{proof}	
	\begin{remark}
		Since $\dot{M}(\gencoord)$ is symmetric (cref. \eqref{eq:mass_skewsem}), another way of stating the skew-symmetric property is to write %
		\begin{align}
			\dot{M}(\gencoord) = C_1(\gencoord, \dot{\gencoord}) + C_2(\gencoord, \dot{\gencoord}) + \left[C_1(\gencoord, \dot{\gencoord}) + C_2(\gencoord, \dot{\gencoord})\right]^\top.
		\end{align}
		Owing to the symmetry of the right-hand-side (rhs), we have $\dot{M} = 2(C_1+ C_2)$.
		\label{rem:skewsem}
	\end{remark}
	\begin{property}[Linearity-in-the-parameters]
		There exists a constant vector $\Theta \in \bb{R}^{l}$ and an $N\times l$ dimensional regressor function $Y(\gencoord, \dot{\gencoord}, \ddot{\gencoord}) \in \bb{R}^{N\times l}$ such that
		\begin{align}
			\begin{split}
				M(q) \ddot{q} &+ \left[C_1 (q, \dot{q}) + C_2(q, \dot{q}) +  D(q, \dot{q})\right]\dot{q} - F(q) -  \\ 
				& \qquad \qquad + 
				N(q)\Adj_{\conf_r}^{-1} \mc{G} = Y(\gencoord, \dot{\gencoord}, \ddot{\gencoord}) \Theta.
			\end{split}
		\end{align}
		\label{ppty:linearity}
	\end{property}
	%
	\begin{proof}[Proof of Property \ref{ppty:linearity}]
		Consider the generalized constitutive law for the full Cosserat model 
		derived in~\cite[\S 6.3]{Boyer2017}. 
		The reduced Lagrangian density in $\liealgebra$ per unit of deformed volume (for all configurations)\footnote{Note that the full Lagrangian density of the soft multisection manipulator is $L_m = \int_0^L \mathfrak{L}(\conf, \twist, \strain) dX$.} for a configuration $\mc{B}$ is  $	\mathfrak{L} = \mathfrak{T} - \mathfrak{U}$~\cite{Boyer2017} where 	$\mathfrak{T}, \mathfrak{U}$ respectively denote the volume's left-reduced kinetic and elastic potential energy densities in $\mc{B}$. From the Euler-Lagrange equation, we have 
		\begin{align}
			  \tau_n=\dfrac{d}{dt}\dfrac{\partial \mathfrak{T}}{\partial \dot{\strain}_n} - \dfrac{\partial\mathfrak{T}}{\partial \strain_n}  + \dfrac{\partial \mathfrak{U}}{\partial\mathfrak{\twist}_n}, \quad n=1,\ldots,N.
			\label{eq:euler_lagrange}
		\end{align}
		Suppose that the material mid-surface crosses the microstructures $\mc{M}_i$ which correspond to the mass center, then the kinetic energy density per unit of deformed area and its rate of change w.r.t $\strain$ are~\cite{Boyer2017}
		\begin{align}
			\mathfrak{T}(\strain) = \frac{1}{2}\bigg\langle \left(\begin{array}{c}
				\omega \\ \nu,
			\end{array}\right), \left(\begin{array}{c}
			\bar{I}  \omega \\ \bar{\rho} \nu 
			\end{array}\right)\bigg\rangle \\
			\mathfrak{T}(\dot{\strain}) = \frac{1}{2}\bigg\langle \left(\begin{array}{c}
				\dot{\omega} \\ \dot{\nu},
			\end{array}\right), \left(\begin{array}{c}
				\bar{I}  \dot{\omega} \\ \bar{\rho} \dot{\nu} 
			\end{array}\right)\bigg\rangle 
		\end{align}
		where $\bar{\rho}$ and $\bar{I}$ respectively denote the mass and angular inertia density per unit volume. It follows that 
		\begin{align}
			\partial_\strain \mathfrak{T} = \left(\begin{array}{c}
				\bar{I}  \omega \\ \bar{\rho} \nu 
			\end{array}\right), \quad 
			\partial_{\dot{\strain}} \mathfrak{T} = \left(\begin{array}{c}
			\bar{I}  \dot{\omega} \\ \bar{\rho} \dot{\nu} 
			\end{array}\right).
			\label{eq:lagrangian_kinetic}
		\end{align}
		In a similar vein, the left invariant density of internal energy $\mathfrak{U}$ is~\cite{Boyer2010}
		\begin{align}			
			\mathfrak{U}(\twist) =  \langle \mc{F}_{int}, (\twist - \twist^d) \rangle
		\end{align}
		for a desired $\twist^d$ and field of internal force constraints $\mc{F}_{int}: \abscissa \in [0, L] \rightarrow \mc{F}_{int}(\abscissa) \in \liealgebra$.
		The potential energy per unit of metric area of the deformed surface (assuming that it is concentrated in the mid-surface) is~\cite{Boyer2017}
		\begin{align}
			{\partial_\twist \mathfrak{U}} = \left(\begin{array}{c}
				\partial \mathfrak{U}/\partial \gamma \\ \partial \mathfrak{U}/\partial \varepsilon
			\end{array}\right) - \left(\begin{array}{c}
				0 \\ \varepsilon
			\end{array}\right) \mathfrak{L}
			\label{eq:lagrangian_potential}
		\end{align}
		so that the Euler-Lagrange equation \eqref{eq:euler_lagrange} becomes
		\begin{align}
			\left(\begin{array}{c}
			\bar{I}  \ddot{\omega} + \dot{\bar{I}} \dot{\omega}\\ \bar{\rho} \ddot{\nu} + \dot{\bar{\rho}} \dot{\nu} 
			\end{array}\right) - \left(\begin{array}{c}
			\bar{I}  \omega \\ \bar{\rho} \nu 
			\end{array}\right) + \left(\begin{array}{c}
			\partial \mathfrak{U}/\partial \gamma \\ \partial \mathfrak{U}/\partial \varepsilon
			\end{array}\right) - \left(\begin{array}{c}
			0 \\ \varepsilon
			\end{array}\right) \mathfrak{L}. 
		\end{align}
		%
\noindent	\textbf{Observe}:  Under the PCS assumption, each microsolid is fixed so that the energy density per unit section of metric volume is akin to that of  a rigid body.  The kinetic and potential energies for the PCS model per section $i$ of $N$ sections  becomes 
		\begin{align}
			\mathfrak{T} = \frac{1}{2}\sum_{i=1}^N \bigg\langle \left(\begin{array}{c}
				^{i+1}\omega_i \\ ^{i+1}\nu_i,
			\end{array}\right), \left(\begin{array}{c}
				^{i+1}\bar{I}_i \,\, ^{i+1}\omega_i \\ ^{i+1}\bar{\rho}_i \,\,^{i+1}\nu_i
			\end{array}\right)\bigg\rangle
		\end{align} 
		where $^{i+1}\omega_i$ is the angular velocity of section $i+1$ in the frame of section $i$, $^{i+1}\nu_i$ is the linear velocity of section $i+1$ in the frame of section $i$ e.t.c. Similarly, the sectional potential energies are
		\begin{align}
			\mathfrak{U}(\twist_i) = \sum_{i=1}^N \langle \{\mc{F}_{int}\}_i, (\twist_i - \twist_i^d) \rangle.
		\end{align} 
	Thus, the kinetic and potential energy are each linear in configuration parameters so that 
	\begin{align}
		\mathfrak{T} = \sum_{i=1}^{N} \dfrac{\partial \mathfrak{T}}{\partial \Sigma_i}\Sigma_i = \sum_{i=1}^{N} \Gamma \mathfrak{T}_i \Sigma_i, \nonumber\\
		\mathfrak{U} = \sum_{i=1}^{N} \dfrac{\partial \mathfrak{U}}{\partial \Sigma_i}\Sigma_i = \sum_{i=1}^{N} \Gamma \mathfrak{U}_i \Sigma_i
		\label{eq:lip}
	\end{align}
	where $\Sigma_i$ is an inertial parameter, $\Gamma \mathfrak{T}_i$ is a function of $\gencoord, \dot{\gencoord}$ and $\Gamma \mathfrak{U}_i$ is a function of $\gencoord$.  	Using \eqref{eq:lip} and plugging \eqref{eq:lagrangian_kinetic} and \eqref{eq:lagrangian_potential} into \eqref{eq:euler_lagrange}, we conclude that the sectionalized piecewise Cosserat dynamics is also linear-in-the-inertial-parameters, given as $\tau(\gencoord) = Y(\gencoord, \dot{\gencoord}, \ddot{\gencoord}) \Theta$,
	%
	where $Y(\gencoord, \dot{\gencoord}, \ddot{\gencoord})$ is the matrix function of $\gencoord, \dot{\gencoord}, \ddot{\gencoord}$ and $\Theta$ is the matrix of parameters.
	\end{proof}
\end{theorem}
	\section{Multivariable Control}
\label{sec:pd_pid_controls}
%
%
\begin{figure}[tb!]
	\centering
	\includegraphics[width=.95\columnwidth, height=.5\columnwidth]{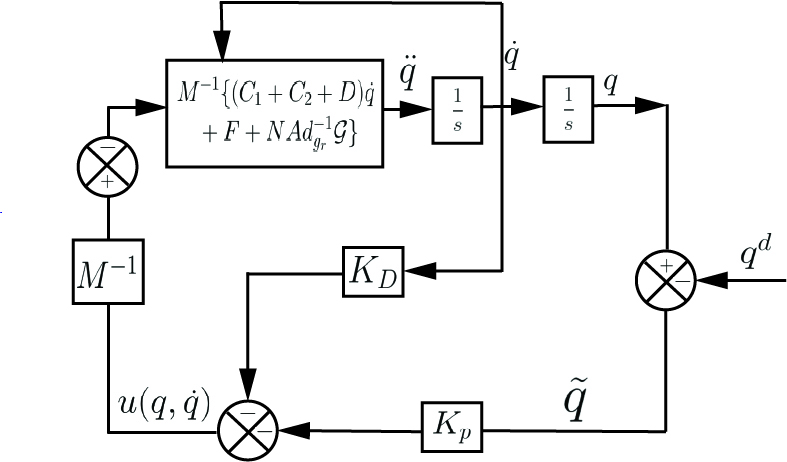}
	\caption{\footnotesize{PD computed torque control block diagram}.}
	\label{fig:block_diag}
\end{figure}
We will control the robot's dynamics by exploiting the properties of \S \ref{sec:structure}. 


 \subsection{Proportional-(Integral)-Derivative Control, \textit{Redux}.}
%
We first show that regarding the generalized torque $\tau(\gencoord)$ as a control input, $u(\gencoord, \dot{\gencoord})$, linear feedback laws are (almost) sufficient for attaining a desired joint configuration (without gravity compensation) in the control law. 
%
\begin{theorem}[Cable-driven Actuation]
	For a set of gains $K_D$ and $K_p$ which are damping and positive diagonal matrices, the control law  (without gravity/buoyancy compensation) 
	\begin{align}
		u(\gencoord, \dot{\gencoord}) = -K_p \tilde{\gencoord} -K_D \dot{\gencoord} - F(\gencoord) 
		\label{eq:pd_controller_no_grav}
	\end{align} 
	under a cable-driven actuation globally asymptotically stabilizes (GAS) system \eqref{eq:newton-euler}, where $\tilde{\gencoord}(t) = \gencoord(t) - \gencoord^d$ is the joint error vector for a desired equilibrium point in space $\gencoord^d$.
\end{theorem} 
\begin{proof}Without gravity, the term $N(\gencoord) \Adj_{g_r}^{-1} \mc{G} = 0$. Let
	$$\breve{C}(\gencoord, \dot{\gencoord}) = C_1 (\gencoord, \dot{\gencoord}) + C_2(\gencoord, \dot{\gencoord})$$
	and write \eqref{eq:newton-euler} for an arbitrary control input $u(\gencoord)$ as 
	\begin{align}
		\begin{split}
			M(q) \ddot{\gencoord} &= u (\gencoord, \dot{\gencoord}) +  F(q)  -\left[\breve{C}(\gencoord, \dot{\gencoord})+ D(q, \dot{q})\right]\dot{q}.
		\end{split}
		\label{eq:newton-euler-u}
	\end{align}
	
	Consider the Lyapunov candidate function
	\begin{align}
		V(\gencoord) = \frac{1}{2} \dot{\gencoord}^\top M(\gencoord) \dot{\gencoord} + \frac{1}{2} \tilde{\gencoord}^\top K_p \tilde{\gencoord}.
		\label{eq:lyap_candidate}
	\end{align}
	\noindent {Observe}: $V(\gencoord)  > 0 \,\, \forall  \,\,\{\gencoord, \gencoord^d\} \setminus (\gencoord = \gencoord^d), \dot{\gencoord}\neq 0$; and $V(\gencoord) = 0\,\, \forall \{\gencoord, \gencoord^d\}$ when $(\gencoord = \gencoord^d)$ and $\dot{\gencoord} = 0$ in the joint space. 
	Differentiating $V(\gencoord)$ yields 
	\begin{align}
		\dot{V}(\gencoord) &= \dot{\gencoord}^\top M(\gencoord) \ddot{\gencoord} + \frac{1}{2} \dot{\gencoord}^\top \dot{M}(\gencoord) \dot{\gencoord} + \tilde{\gencoord}^\top K_p \dot{\tilde{\gencoord}}, \nonumber \\
		&=\dot{\gencoord}^\top \left(u(\gencoord, \dot{\gencoord}) + F(\gencoord) - \left[\breve{C} (\gencoord, \dot{\gencoord}) + D(\gencoord, \dot{\gencoord})\right]\dot{\gencoord} \right)  \nonumber \\
		& \qquad + \frac{1}{2} \dot{\gencoord}^\top \dot{M}(\gencoord) \dot{\gencoord} + \tilde{\gencoord}^\top K_p \dot{\gencoord}, \nonumber \\
		&=\dot{\gencoord}^\top \left(u(\gencoord, \dot{\gencoord}) + F(\gencoord) -  \left[\breve{C} (\gencoord, \dot{\gencoord}) + D(\gencoord, \dot{\gencoord})\right]\dot{\gencoord} \right)   \nonumber \\
		& +\tilde{\gencoord}^\top K_p \dot{\gencoord}  + \dot{\gencoord}^\top\left[ \frac{1}{2} \left(\dot{M}(\gencoord) - 2 \breve{C} (\gencoord, \dot{\gencoord})\right) + \breve{C}(\gencoord, \dot{\gencoord})\right]\dot{\gencoord}, \nonumber \\
		&=\dot{\gencoord}^\top \left[u(\gencoord, \dot{\gencoord}) + F(\gencoord) + K_p \tilde{\gencoord} - D(\gencoord, \dot{\gencoord})\dot{\gencoord} \right]
		\label{eq:lyap_deriv_pd_no_grav}
	\end{align}
	where the last line follows from the skew symmetric property established in Remark \ref{rem:skewsem}. Suppose we choose the control law \eqref{eq:pd_controller_no_grav}, then
	\begin{align}
		\dot{V}(\gencoord) = - \dot{\gencoord}^\top \left[K_D  +D(\gencoord, \dot{\gencoord}) \right]\dot{\gencoord} \le 0
		\label{eq:lyap_proof_pd}
	\end{align} 
	if we choose $K_D$ as a block diagonal matrix of positive gains since the drag term $D(\gencoord, \dot{\gencoord}) > 0$. Thus, we have that the potential function is decreasing if $\gencoord$ is non-zero. However, asymptotic stability of trajectories as a result of the potential (Lyapunov) function candidate does not automatically follow since $\dot{V}(t)$ is not negative at $\dot{q}=0$. That is, it is possible for the manipulator's joint space variable derivative $\dot{q} = 0$ when $\gencoord \neq {\gencoord}^d$. 
	
	To prove global asymptotic stability, suppose the domain $\Omega \subset D \in \bb{R}^{6N\times 6N}$ is the compact, positively invariant domain with respect to \eqref{eq:newton-euler}. Let $\mc{E}$ be the set of all $\gencoord \in \Omega$ where  $\dot{V}$ is identically zero so that \eqref{eq:lyap_proof_pd} implies that $\dot{q} = 0$ and $\ddot{q} = 0$. From \eqref{eq:newton-euler-u}, we must have 
	\begin{align}	
		\begin{split}
			M(q) \ddot{\gencoord}  &+\left[\breve{C}(\gencoord, \dot{\gencoord})+ D(q, \dot{q})\right]\dot{q} = -K_p \tilde{\gencoord} -K_D \dot{\gencoord}  
		\end{split}
	\end{align}
	which implies that $0 = -K_p \tilde{\gencoord}$ and hence $\tilde{\gencoord} = 0$. If $\Upsilon$ is the largest invariant set of $\mc{E}$, then by LaSalle's invariance theorem~\cite{Khalil}, every solution starting in $\Omega$ approaches $\Upsilon$ as $t \rightarrow \infty$. Whence, the equilibrium is GAS.
\end{proof}

\begin{corollary}[Fluid-driven actuation]
	If the robot is operated without cables, and is driven with a dense medium such as pressurized air or water, then the term $F(\gencoord)=0$ so that the control law 
	\begin{align}
		u(\gencoord, \dot{\gencoord}) = -K_p \tilde{\gencoord} -K_D \dot{\gencoord} 
		\label{eq:pd_controller_fluidic}
	\end{align} 
	 globally  asymptotically stabilizes the system.
\end{corollary}

\begin{proof}
	The proof follows easily since $F(\gencoord) = 0$  in \eqref{eq:lyap_deriv_pd_no_grav}.
\end{proof}


\begin{theorem}[Gravity/Buoyancy-Compensated Control]
	Without gravity and/or buoyancy-compensation, there are parasitic disturbances and unmodeled parameters in \eqref{eq:newton-euler} where an offset in the positional error $\tilde{\gencoord}$ at steady state may show up since we do not explicitly compensate for gravity and buoyancy.where the PD control law is not sufficient for driving the strain states to equilibrium. The following controller will achieve zero steady state convergence.
		\begin{align}
			u(\gencoord, \dot{\gencoord}) = -K_p \tilde{\gencoord} -K_D \dot{\gencoord} -F(\gencoord) - N(\gencoord) \Adj_{g_r}^{-1} \mc{G}.
			\label{eq:control_grav_buoy}
		\end{align}
\end{theorem}
\begin{proof}
	From \eqref{eq:newton-euler}, write 
	\begin{align}
		\begin{split}
			M(q) \ddot{\gencoord} = u (\gencoord, \dot{\gencoord}) +  F(q) + N(q) \Adj_{g_r} ^{-1}\mathcal{G}  \\
			-\left[\breve{C}(\gencoord, \dot{\gencoord})+ D(\gencoord, \dot{\gencoord})\right]\dot{q}.
		\end{split}
		\label{eq:newton_euler_grav_comp}
	\end{align}
	Let us choose the Lyapunov function candidate 
	\begin{align}
		\label{eq:lyap_pid_grav}
		V(\gencoord) = \frac{1}{2} \dot{\gencoord}^\top M(\gencoord) \dot{\gencoord} + \frac{1}{2} \tilde{\gencoord}^\top K_p \tilde{\gencoord}  +    \frac{1}{2} \int_0^t{\tilde{\gencoord}^\top(s)K_I \tilde{\gencoord} ds},
	\end{align}
	where  $V(\gencoord) > 0 \,\, \forall  \,\,\{\gencoord, \gencoord^d\} \setminus (\gencoord = \gencoord^d), \dot{\gencoord}\neq 0$;  $V(\gencoord) = 0\,\, \forall \{\gencoord, \gencoord^d\}$ when $(\gencoord = \gencoord^d)$,  $K_I$ is the block diagonal integral gain, and $\dot{\gencoord} = 0$ in the joint space. 
	Tt can be verified that
	\begin{align}
		\label{eq:lyap_deriv_pid_grav}
		\dot{V}(\gencoord) 
		%
		%
		&=\dot{\gencoord}^\top \left[u(\gencoord, \dot{\gencoord}) + F(\gencoord)  + \left\{K_p + \frac{K_I}{s}\right\} \tilde{\gencoord} \nonumber \right. \\
		& \left. \qquad \qquad - D(\gencoord, \dot{\gencoord})\dot{\gencoord} + N(q) \Adj_{g_r} ^{-1}\mathcal{G}\right]
	\end{align}
	where we have used the frequency domain representation of the integral operator \ie $1/s$. 
	If we set 
	\begin{align}
		u(\gencoord, \dot{\gencoord}) &= - \left\{K_p + \frac{K_I}{s}\right\} \tilde{\gencoord}  - K_D \dot{q}   - F(\gencoord) - N(\gencoord) \Adj_{g_r}^{-1} \mc{G},
		\label{eq:cable_control_grav_all}
	\end{align}
	we must have
	\begin{align}
		\dot{V}(t) = -\dot{\gencoord}^\top \left[K_D + D(\gencoord, \dot{\gencoord}) \right] \gencoord \le 0, \text{ if } \dot{q} \neq 0
	\end{align}
	else for $\dot{V}$ identically zero, LaSalle's principle applies. 
\end{proof}

 Suppose that we choose appropriate proportional and derivative gains with the gravity-compensation term in \eqref{eq:cable_control_grav_all},  then the integral again $K_I$ may be ignored as 
	\begin{align}
		u(\gencoord, \dot{\gencoord}) &= - K_p \tilde{\gencoord}  - K_D \dot{q}   - F(\gencoord) - N(\gencoord) \Adj_{g_r}^{-1} \mc{G}.
	\end{align} 
	The proof of the theorem follows as a result.
	\section{Numerical Setup and Results}
\label{sec:results}
\begin{figure*}[t]
	\centering
	\begin{tabular}{cc} 
		\textbf{a1)} Cable-driven, strain twist setpoint terrestrial control. & \textbf{a2)} Fluid-actuated, strain twist setpoint  terrestrial control. 
		\\
		\includegraphics[width=0.48\textwidth]{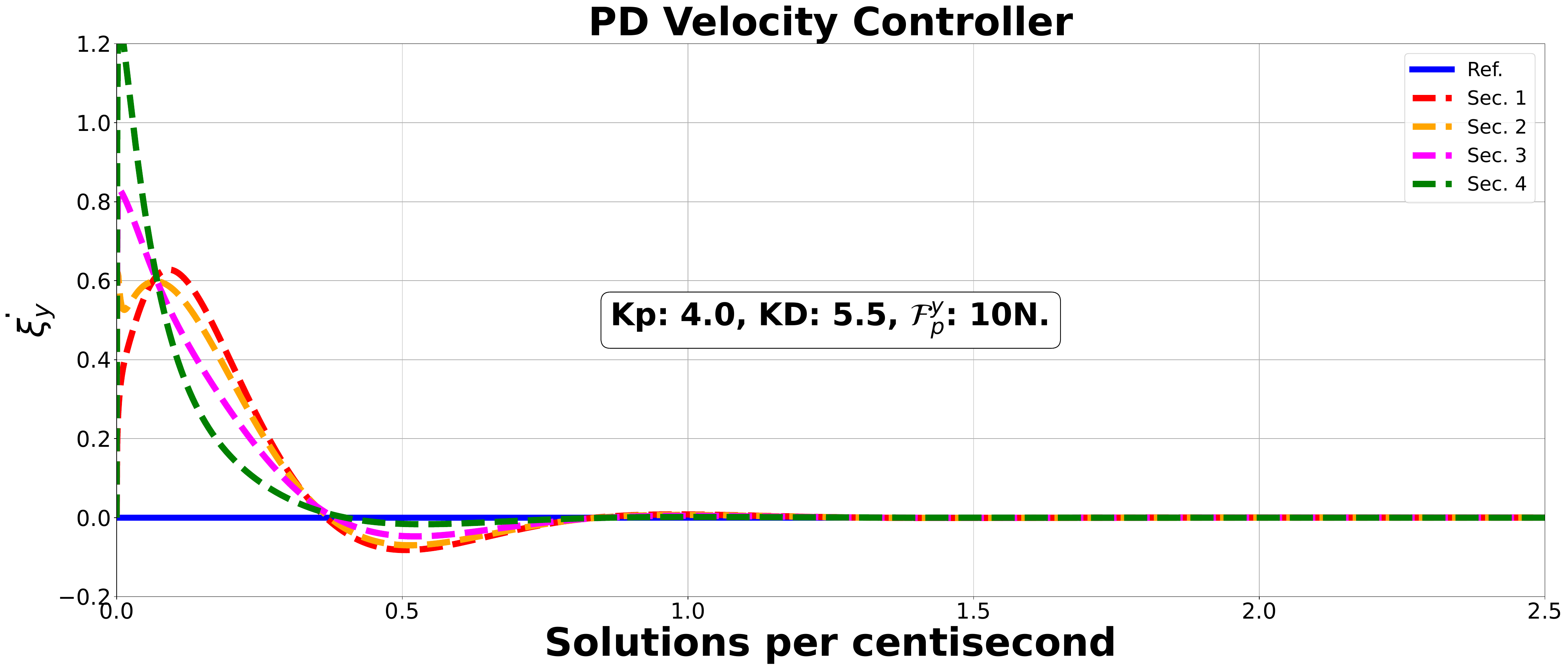}
		& 
		\includegraphics[width=0.48\textwidth]{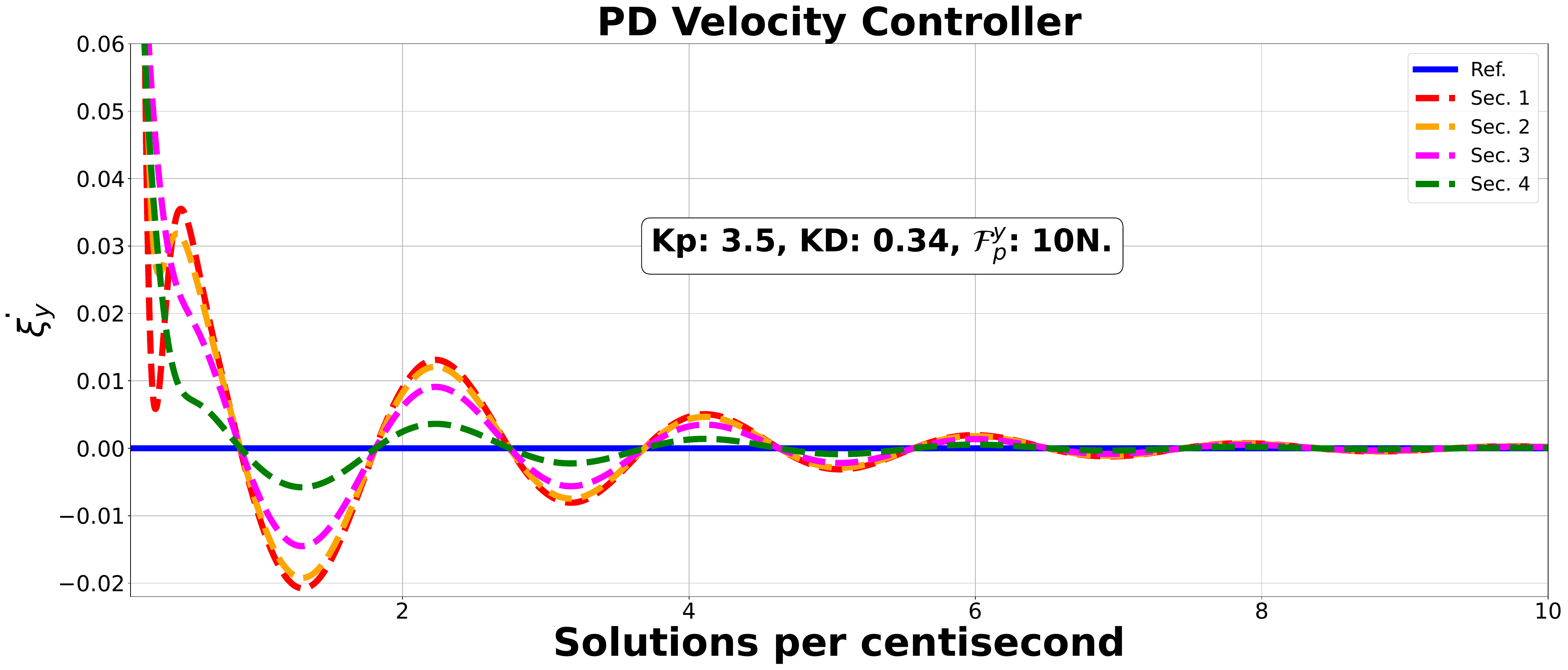}
		\\
		\textbf{a3)} Fluid-actuated, strain twist setpoint underwater control. & \textbf{a4)} Cable-driven, strain twist setpoint  regulation.
		\\		
		\includegraphics[width=0.48\textwidth]{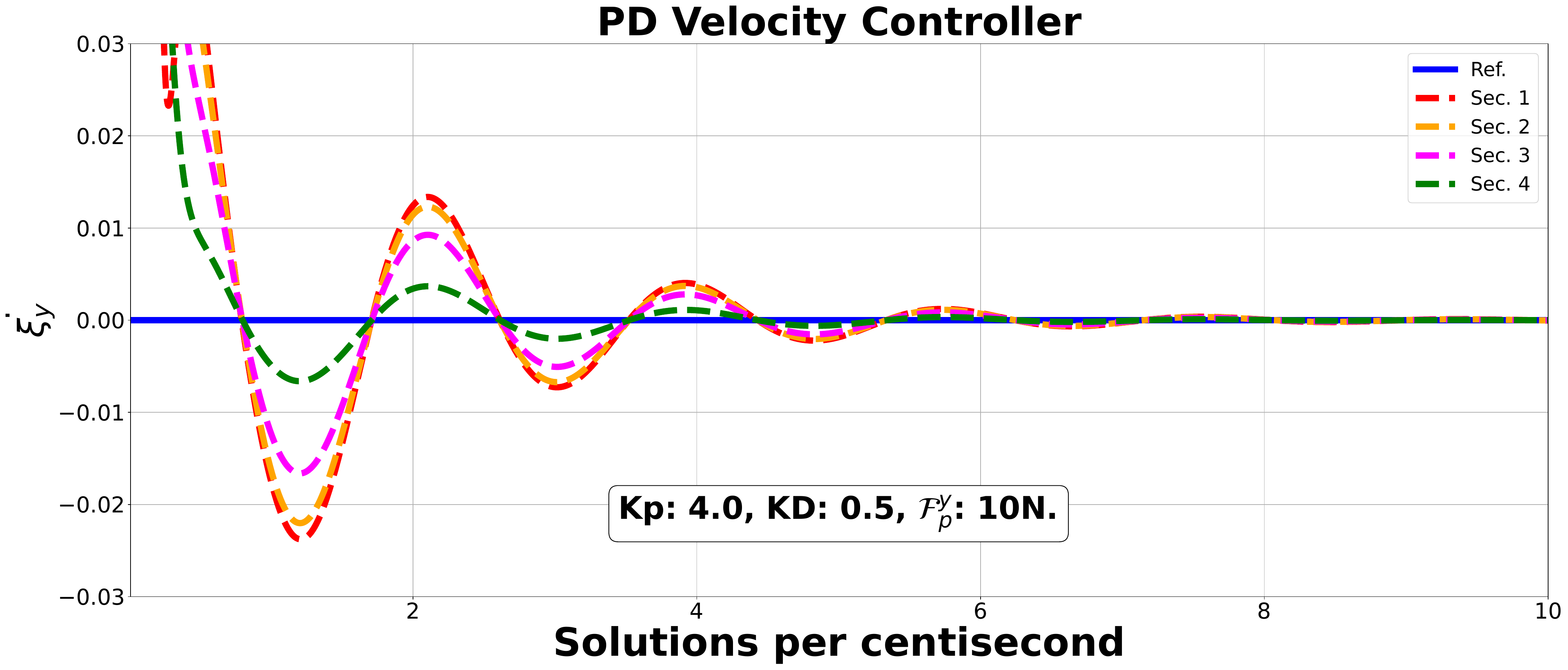} &
		\includegraphics[width=0.48\textwidth]{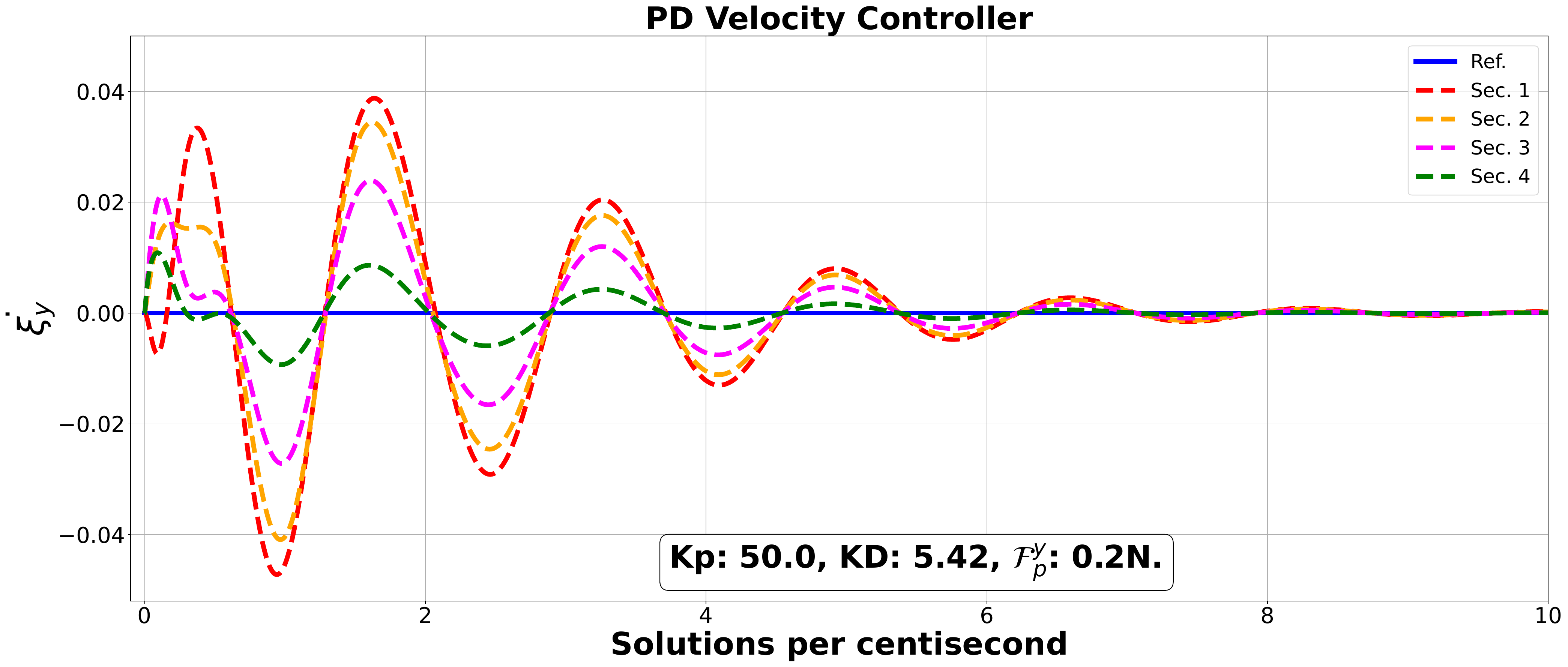}
		\\
		\textbf{b1)} Position control with no gravity-compensation. & \textbf{b2)} Gravity-compensated terrestrial position control.
		\\
		\includegraphics[width=0.48\textwidth]{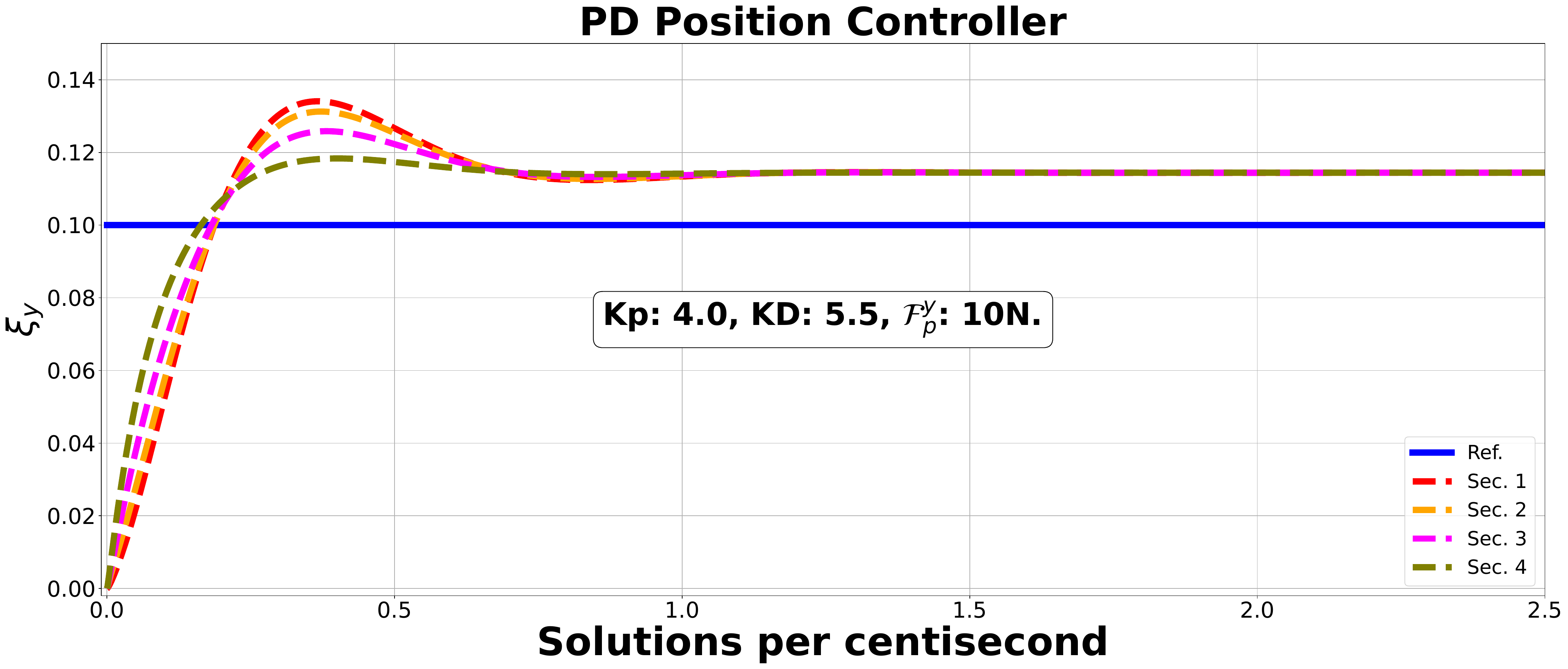}
		&
		\includegraphics[width=0.48\textwidth]{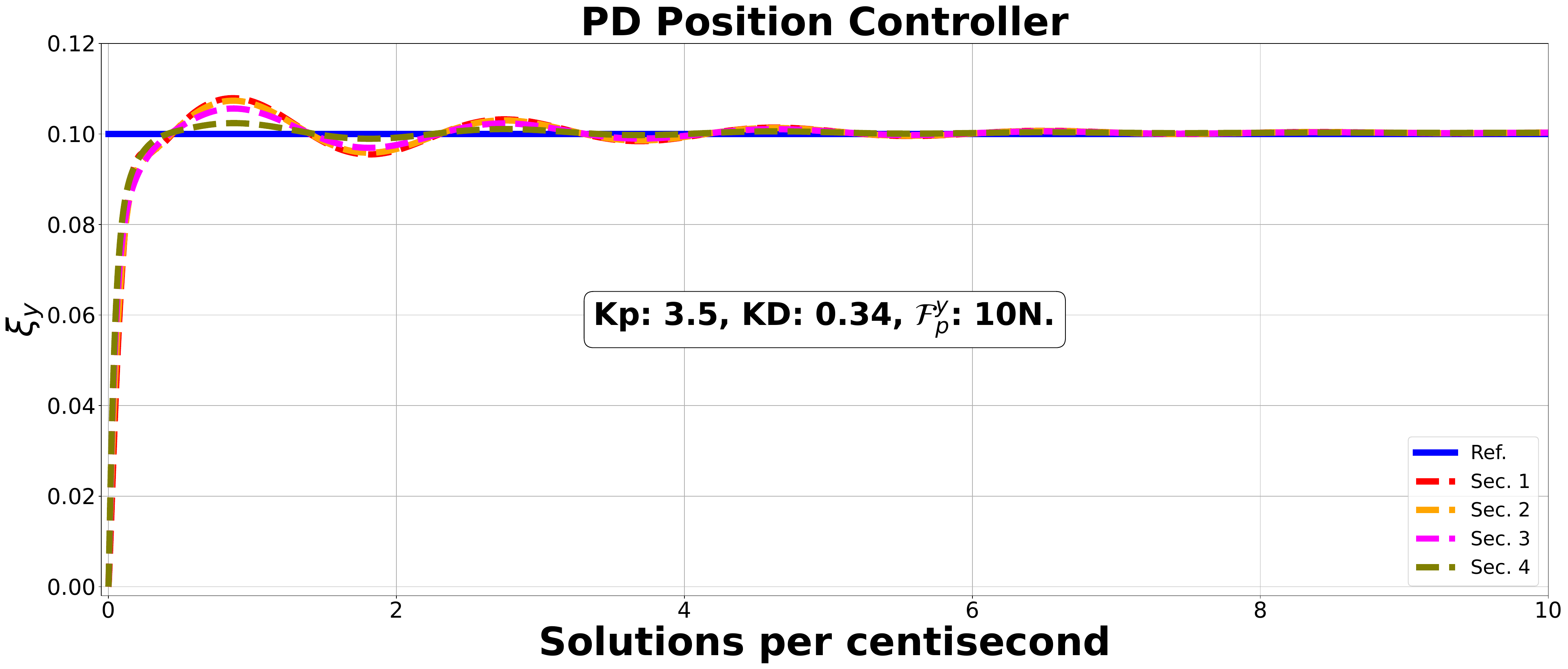}%
	\end{tabular}
	\caption{\textit{(a1 -- a3):} Linear strain twist  state regulation for a 4-section (41 microsolid pieces per section) soft arm under a $10N$ lateral tip load, $\mc{F}_p^y$ that is \textit{(a1)} cable-driven in a sparse medium \textit{(a2)} fluid-actuated in a sparse medium. \textit{(a3)} fluid-actuated in a dense medium such as water (\ie with drag forces compensation).  \textit{(a4)} Under a miniature tip-load of $0.2N$, a cable-driven 4-section arm finely regulates strain twists to equilibrium over time. \textit{(b1-b2)}: Linear strain position regulation for a cable-driven arm operating in air \textit{(b1)} and a fluid-actuated arm illustrating the effect of steady state errors in the absence of gravity compensation. 
		The horizontal axes show the number of (re)-integration time-steps per second for the adaptive Runge-Kutta-Fehlberg integrator we utilized in computing the controllers.
	}
	\label{fig:pd_controls}
\end{figure*}
\begin{figure}[tb!]
	\centering 
	\begin{tabular}{cc} 
		\textbf{(c1)} \\ \includegraphics[width=0.47\textwidth]{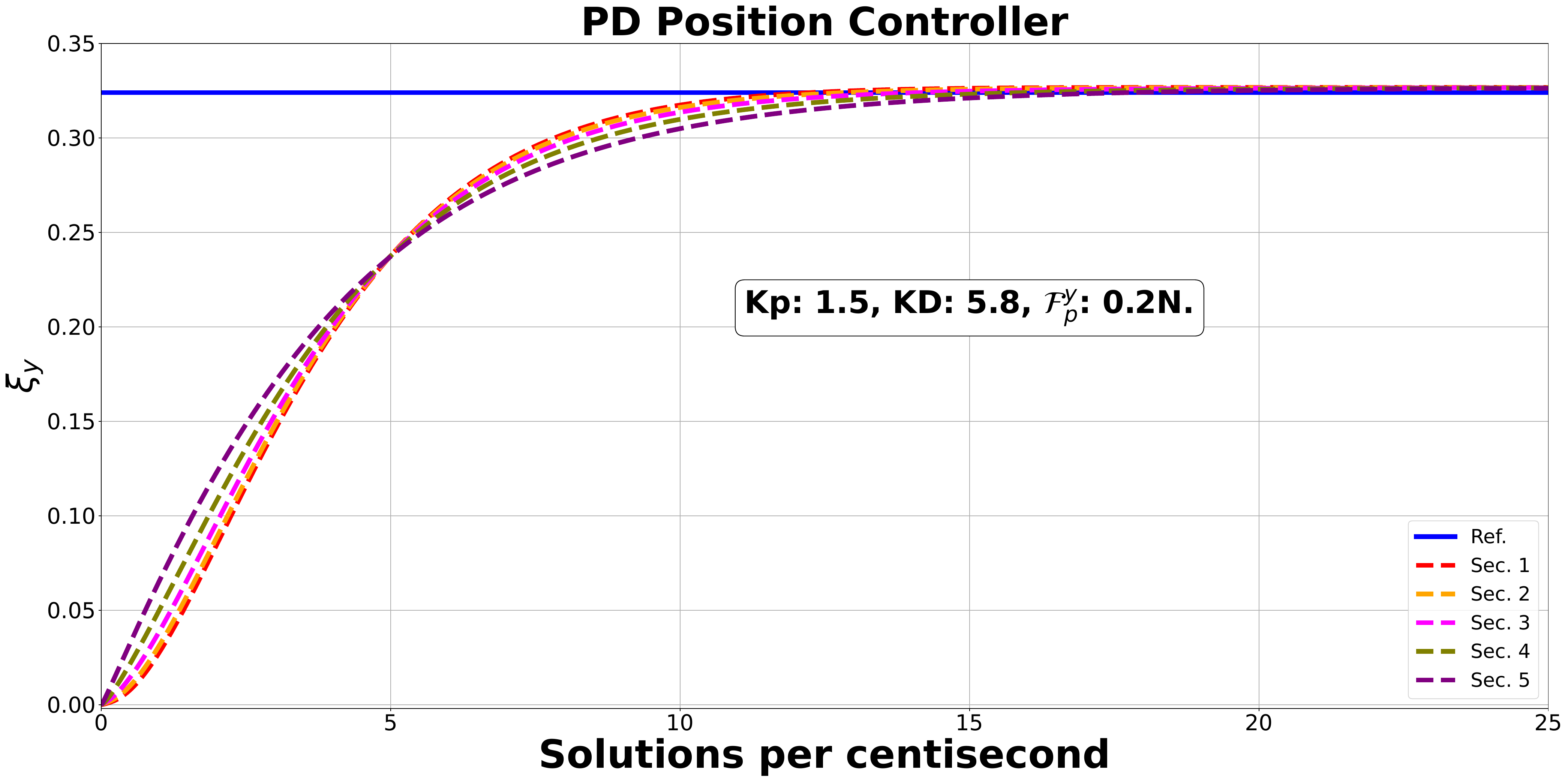} 
		\\
		\\
		\textbf{(c2)} \\ \includegraphics[width=0.47\textwidth]{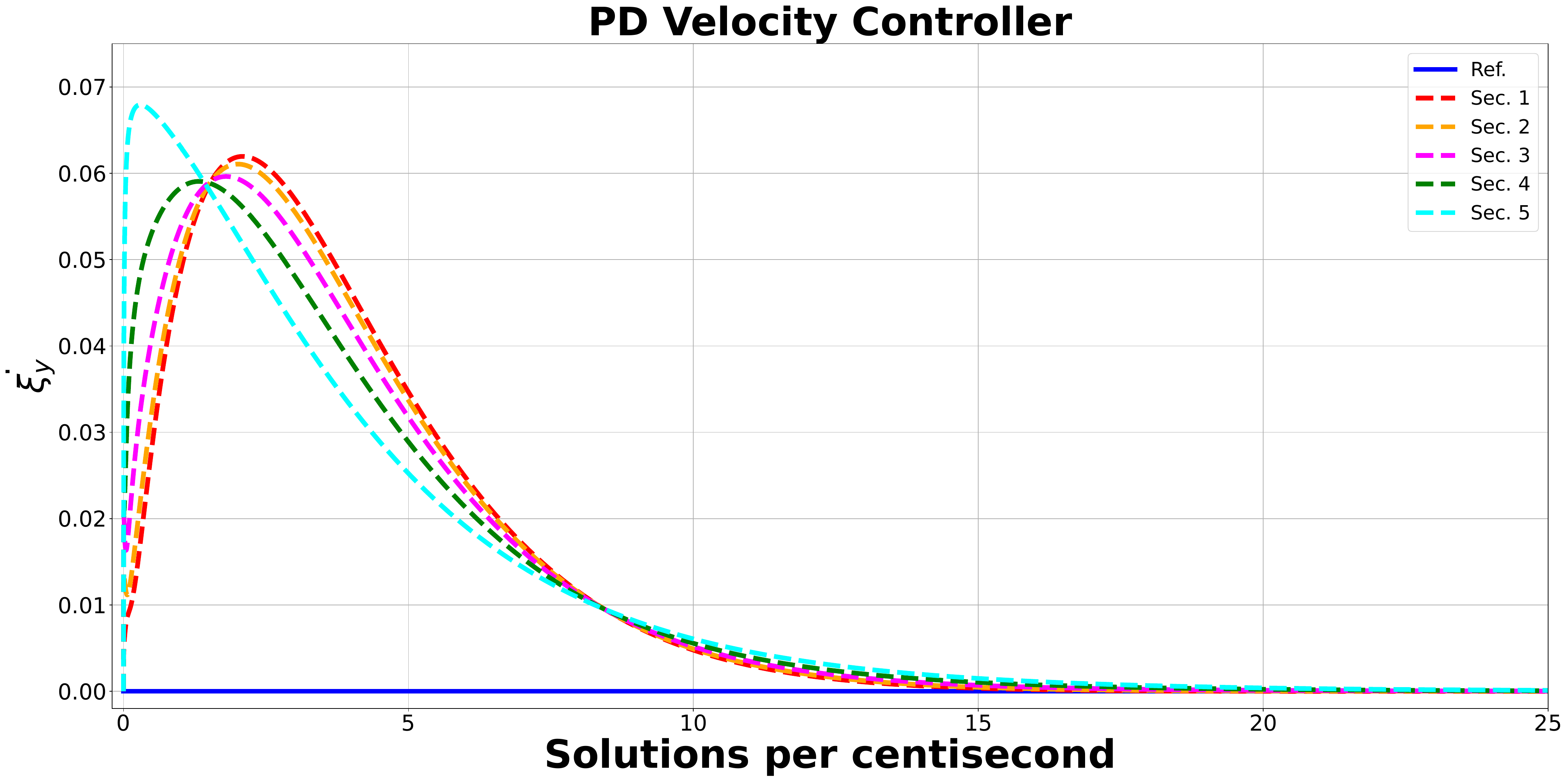}%
		\\
	\end{tabular}
\caption{Gravity-based compensation position and velocity control: \textbf{c1)} Cable-driven gravity-compensated position control. \textbf{c2)} Cable-driven gravity-compensated velocity control.}
\end{figure}
Our goal is to regulate the strain and strain velocity states of the robot per section under different constant tip loads \textit{despite the inevitable non-constant loads due to gravity, external forces, and inertial forces}. 

\subsection{System Setup and Parameters}
As seen in \autoref{fig:pcs_kine}, the tip load acts on the $+y$-axis in the robot's base frame. 
We use $\mc{F}_p^y$ to represent the tip load acting along the $+y$ direction in what follows. 
Given the geometry of the robot, we choose a drag coefficient of $0.82$ (a Reynolds number of order $104$) for underwater operations. We set the Young's modulus, $E = 110 kPa$ and the shear viscosity modulus to $3 kPa$. The bending second inertia momenta are set to $I_y = I_z =\pi r^4/4$ while the torsion second moment of inertia is set to $I_x=\pi r^4/2$ for $r=0.1 m$, the arm's radius -- uniform across sections. The arm length is $L=0.2m$. We assume a (near-incompressible) rubber material makes up the robot's body and set its Poisson ratio to $0.45$; the mass is chosen as $\mc{M} = \rho [I_x, I_y, I_z, A, A, A]$ for a cylindrical soft shell's nominal density of  $\rho=2,000 kg m^{-3}$ as used in~\cite{RendaTRO18}; the cross-sectional area $A=\pi r^2$ so that $I_x=\pi r^4/2$. The drag screw stiffness matrix $D$ in \eqref{eq:newton-euler} is a function of each section's geometry and hydrodynamics so that $D = -\rho_w \nu^T \nu \breve{D} \nu/|\nu|$ where $\rho_w$ is the  water density set to $997 kg/m^3$, and $\breve{D}$ is the tensor that models the geometry and hydrodynamics factors in the viscosity model (see \cite[\S II.B, eq. 6]{RendaTRO18}). The curvilinear abscissa, $X \in [0, L]$ was discretized into $41$ microsolids per section.

\subsection{Discussion}
We adopted the recursive articulated-body algorithm~\cite{Featherstone} and integrate the right-hand-side of differential equations  \eqref{eq:newton-euler-u} and \eqref{eq:newton_euler_grav_comp} using a Runge-Kutta-Fehlberg (RKF) adaptive scheme implemented in Python and Torch~\cite{pytorch} with relative and absolute errors respectively set to $10^{-7}$ and $10^{-9}$. We have found these tolerance values to be crucial for a successful numerical integration scheme as it avoids numerical instability. Depicted on the vertical axes of each chart of \autoref{fig:pd_controls} are the  strain positions or twists along $+y$-direction (on the robot's local frame) while the horizontal axes depict the number of adaptive RKF (re-) integration steps per for every tenth of a second. The trajectory evolution over time per discretized Cosserat section is shown in the various ``dash-dotted" lines,  while the ``solid blue" lines denote the reference. The controller parameters are annotated within chart together with the amount of constant tip load in Newtons. We see that all joint configurations are stabilized to reference strain twist states  $(\dot{\gencoord}^d)$ for respective constant tip loads. It can also be inferred that more sections a in the discretized Cosserat model lead to less bumpy state regulation strategy.  

In (\textit{a1--a3}), the strain twists are precisely regulated to zero offset errors  despite large constant tip loads (\textit{a1-a3}); and with small tip disturbances \textit{(a4)}. Notably, the large value of the proportional gain causes overdamping in the transients before convergence. 
However,  a 20\% error offset at steady state is observed in the linear strain position regulation case \textbf{(b1)} in the absence of the gravity compensation term in the control law. We attribute this to the nature of PD Lagrangian controllers realized via computed torque methods~\cite{SpongHist}. For accurate position control with PD controllers, steady state errors can  be achieved when the gravity os compensated in the control design as  \textit{(b2)} shows. 
	\section{Conclusion}
\label{sec:conclude}
We have presented the Lagrangian properties for soft robots under a discrete Cosserat model. These properties were then exploited to cancel out nonlinearities in the derived controllers  for strain states regulation. Our numerical experiments confirm the conclusions from our Lyapunov analyses. Similar to rigid robots under PD control laws with Lagrangian dynamics, we have observed strain position steady state offsets. Position control efficacy are shown to improve if gravity and/or buoyancy is compensated in the control law. 

	\section{Acknowledgment}
A vote of thanks to James Forbes for early manuscript feedback. 
	\bibliographystyle{plainnat}
	\bibliography{SoRoBC}

\begin{thebibliography}{23}
\providecommand{\natexlab}[1]{#1}
\providecommand{\url}[1]{\texttt{#1}}
\expandafter\ifx\csname urlstyle\endcsname\relax
  \providecommand{\doi}[1]{doi: #1}\else
  \providecommand{\doi}{doi: \begingroup \urlstyle{rm}\Url}\fi

\bibitem[Agarwal et~al.(2016)Agarwal, Besuchet, Audergon, and Paik]{wearables2}
Gunjan Agarwal, Nicolas Besuchet, Basile Audergon, and Jamie Paik.
\newblock {Stretchable Materials for Robust Soft Actuators Towards Assistive
  Wearable Devices}.
\newblock \emph{Scientific reports}, 6\penalty0 (1):\penalty0 34224, 2016.

\bibitem[Boyer and Renda(2017)]{Boyer2017}
Frederic Boyer and Federico Renda.
\newblock {Poincar{\'{e}}'s Equations for Cosserat Media: Application to
  Shells}.
\newblock \emph{Journal of Nonlinear Science}, 27\penalty0 (1):\penalty0 1--44,
  2017.
\newblock ISSN 14321467.

\bibitem[Boyer et~al.(2010)Boyer, Porez, and Leroyer]{Boyer2010}
Frederic Boyer, Mathieu Porez, and Alban Leroyer.
\newblock Poincar{\'e}--cosserat equations for the lighthill three-dimensional
  large amplitude elongated body theory: application to robotics.
\newblock \emph{Journal of Nonlinear Science}, 20:\penalty0 47--79, 2010.

\bibitem[Boyer et~al.(2022)Boyer, Lebastard, Candelier, Renda, and
  Alamir]{RendaStatics}
Fr{\'e}d{\'e}ric Boyer, Vincent Lebastard, Fabien Candelier, Federico Renda,
  and Mazen Alamir.
\newblock {Statics and Dynamics of Continuum Robots based on Cosserat Rods and
  Optimal Control Theorie}s.
\newblock \emph{IEEE Transactions on Robotics}, 39\penalty0 (2):\penalty0
  1544--1562, 2022.

\bibitem[Coevoet et~al.(2017)Coevoet, Escande, and Duriez]{FEM}
Eulalie Coevoet, Adrien Escande, and Christian Duriez.
\newblock {Optimization-based Inverse Model of Soft Robots with Contact
  Handling}.
\newblock \emph{{IEEE Robotics and Automation Letters}}, 2\penalty0
  (3):\penalty0 1413--1419, 2017.

\bibitem[Cosserat and Cosserat(1909)]{CosseratBrothers}
Eug{\`e}ne Maurice~Pierre Cosserat and Fran{\c{c}}ois Cosserat.
\newblock \emph{Th{\'e}orie des corps d{\'e}formables}.
\newblock A. Hermann et fils, 1909.

\bibitem[Featherstone(2014)]{Featherstone}
Roy Featherstone.
\newblock \emph{Rigid Body Dynamics Algorithms}.
\newblock Springer, 2014.

\bibitem[Hannan and Walker(2003)]{IanWalkerDiffKine}
Michael~W Hannan and Ian~D Walker.
\newblock {Kinematics and the Implementation of an Elephant's Trunk Manipulator
  and other Continuum Style Robots}.
\newblock \emph{{Journal of Robotic Systems}}, 20\penalty0 (2):\penalty0
  45--63, 2003.

\bibitem[Hauser et~al.(2014)Hauser, Füchslin, and
  Pfeifer]{e-book_on_morphological_computation_2014}
Helmut Hauser, Rudolf~M. Füchslin, and Rolf Pfeifer, editors.
\newblock \emph{{Opinions and Outlooks on Morphological Computation}}.
\newblock Hauser, Helmut and Füchslin, Rudolf M. and Pfeifer, Rolf, 2014.
\newblock ISBN 978-3-033-04515-6.

\bibitem[Khalil(2015)]{Khalil}
Hassan~K Khalil.
\newblock \emph{{Nonlinear Control}}, volume 406.
\newblock Pearson New York, 2015.

\bibitem[Laschi et~al.(2012)Laschi, Cianchetti, Mazzolai, Margheri, Follador,
  and Dario]{CeciliaOctopus}
Cecilia Laschi, Matteo Cianchetti, Barbara Mazzolai, Laura Margheri, Maurizio
  Follador, and Paolo Dario.
\newblock Soft robot arm inspired by the octopus.
\newblock \emph{Advanced robotics}, 26\penalty0 (7):\penalty0 709--727, 2012.

\bibitem[Manti et~al.(2015)Manti, Hassan, Passetti, D'Elia, Laschi, and
  Cianchetti]{grippers}
Mariangela Manti, Taimoor Hassan, Giovanni Passetti, Nicol{\`o} D'Elia, Cecilia
  Laschi, and Matteo Cianchetti.
\newblock {A Bioinspired Soft Robotic Gripper for Adaptable and Effective
  Grasping}.
\newblock \emph{Soft Robotics}, 2\penalty0 (3):\penalty0 107--116, 2015.

\bibitem[Marchese et~al.(2014)Marchese, Onal, and Rus]{Marchese}
Andrew~D Marchese, Cagdas~D Onal, and Daniela Rus.
\newblock {Autonomous Soft Robotic Fish Capable of Escape Maneuvers Using
  Fluidic Elastomer Actuators}.
\newblock \emph{Soft robotics}, 1\penalty0 (1):\penalty0 75--87, 2014.

\bibitem[Nuckols et~al.(2021)Nuckols, Lee, Swaminathan, Orzel, Howe, and
  Walsh]{ConorLamb}
R.~W. Nuckols, S.~Lee, K.~Swaminathan, D.~Orzel, R.~D. Howe, and C.~J. Walsh.
\newblock Individualization of exosuit assistance based on measured muscle
  dynamics during versatile walking.
\newblock \emph{Science Robotics}, 6\penalty0 (60):\penalty0 eabj1362, 2021.

\bibitem[Paszke et~al.(2019)Paszke, Gross, Massa, Lerer, Bradbury, Chanan,
  Killeen, Lin, Gimelshein, Antiga, Desmaison, Kopf, Yang, DeVito, Raison,
  Tejani, Chilamkurthy, Steiner, Fang, Bai, and Chintala]{pytorch}
Adam Paszke, Sam Gross, Francisco Massa, Adam Lerer, James Bradbury, Gregory
  Chanan, Trevor Killeen, Zeming Lin, Natalia Gimelshein, Luca Antiga, Alban
  Desmaison, Andreas Kopf, Edward Yang, Zachary DeVito, Martin Raison, Alykhan
  Tejani, Sasank Chilamkurthy, Benoit Steiner, Lu~Fang, Junjie Bai, and Soumith
  Chintala.
\newblock Pytorch: An imperative style, high-performance deep learning library.
\newblock In \emph{Advances in Neural Information Processing Systems 32}, pages
  8024--8035. Curran Associates, Inc., 2019.

\bibitem[Polygerinos et~al.(2015)Polygerinos, Wang, Galloway, Wood, and
  Walsh]{ConorSoroGlove}
Panagiotis Polygerinos, Zheng Wang, Kevin~C Galloway, Robert~J Wood, and
  Conor~J Walsh.
\newblock Soft robotic glove for combined assistance and at-home
  rehabilitation.
\newblock \emph{Robotics and Autonomous Systems}, 73:\penalty0 135--143, 2015.

\bibitem[Qiu et~al.(2023)Qiu, Zhang, Sun, Xiong, Lu, and Wang]{RSSPCC}
Ke~Qiu, Jingyu Zhang, Danying Sun, Rong Xiong, Haojian Lu, and Yue Wang.
\newblock An efficient multi-solution solver for the inverse kinematics of
  3-section constant-curvature robots.
\newblock \emph{arXiv preprint arXiv:2305.01458}, 2023.

\bibitem[Renda et~al.(2018)Renda, Boyer, Dias, and Seneviratne]{RendaTRO18}
Federico Renda, Fr{\'e}d{\'e}ric Boyer, Jorge Dias, and Lakmal Seneviratne.
\newblock Discrete cosserat approach for multisection soft manipulator
  dynamics.
\newblock \emph{IEEE Transactions on Robotics}, 34\penalty0 (6):\penalty0
  1518--1533, 2018.

\bibitem[Romero et~al.(2014)Romero, Ortega, and Sarras]{Ortega2014}
Jos{\'e}~Guadalupe Romero, Romeo Ortega, and Ioannis Sarras.
\newblock A globally exponentially stable tracking controller for mechanical
  systems using position feedback.
\newblock \emph{IEEE Transactions on Automatic Control}, 60\penalty0
  (3):\penalty0 818--823, 2014.

\bibitem[Shih et~al.(2023)Shih, Naughton, Halder, Chang, Kim, Gillette, Mehta,
  and Gazzola]{HierarchicalGazzolla}
Chia-Hsien Shih, Noel Naughton, Udit Halder, Heng-Sheng Chang, Seung~Hyun Kim,
  Rhanor Gillette, Prashant~G Mehta, and Mattia Gazzola.
\newblock Hierarchical control and learning of a foraging cyberoctopus.
\newblock \emph{Advanced Intelligent Systems}, page 2300088, 2023.

\bibitem[Spong(2022)]{SpongHist}
Mark~W. Spong.
\newblock An historical perspective on the control of robotic manipulators.
\newblock \emph{Annual Review of Control, Robotics, and Autonomous Systems},
  5:\penalty0 1--31, 2022.

\bibitem[Thuruthel et~al.(2018)Thuruthel, Falotico, Renda, and
  Laschi]{ThuruthelSoRo}
Thomas~George Thuruthel, Egidio Falotico, Federico Renda, and Cecilia Laschi.
\newblock {Model-based Reinforcement Learning for Closed-loop Dynamic Control
  of Soft Robotic Manipulators}.
\newblock \emph{IEEE Transactions on Robotics}, 35\penalty0 (1):\penalty0
  124--134, 2018.

\bibitem[Yap et~al.(2016)Yap, Kamaldin, Lim, Nasrallah, Goh, and
  Yeow]{wearables1}
Hong~Kai Yap, Nazir Kamaldin, Jeong~Hoon Lim, Fatima~A Nasrallah, James
  Cho~Hong Goh, and Chen-Hua Yeow.
\newblock {A Magnetic Resonance Compatible Soft Wearable Robotic Glove for Hand
  Rehabilitation and Brain Imaging}.
\newblock \emph{IEEE transactions on neural systems and rehabilitation
  engineering}, 25\penalty0 (6):\penalty0 782--793, 2016.

\end{thebibliography}
	\numberwithin{equation}{section}	
	\setcounter{equation}{0}
	\setcounter{lemma}{0}
	\setcounter{theorem}{0}
\end{document}